\documentclass[11pt]{article}

\usepackage[T1]{fontenc}
\usepackage[utf8]{inputenc}

\usepackage{geometry}
\usepackage{amsmath,amssymb,amsthm}
\usepackage{graphicx}
\usepackage[numbers]{natbib}
\usepackage{url} 
\usepackage{hyperref}

\newtheorem{theorem}{Theorem}
\newtheorem{proposition}[theorem]{Proposition}

\title{Stable but Wrong: When More Data Degrades Scientific Conclusions}

\author{
Zhipeng Zhang$^{1,2}$\thanks{Corresponding author: \texttt{zhangzhipeng@chinamobile.com}}
\and
Kai Li$^{1}$
}

\date{
$^{1}$ China Mobile Research Institute, Beijing 100053, China\\
$^{2}$ China Mobile GBA (Greater Bay Area) Innovation Institute, Guangzhou 510656, China\\[1em]
}

\begin{document}
\maketitle

\begin{abstract}
Modern science increasingly relies on ever-growing observational datasets and automated inference pipelines, under the implicit belief that accumulating more data makes scientific conclusions more reliable. Here we show that this belief can fail in a fundamental and irreversible way. We identify a structural regime in which standard inference procedures converge smoothly, remain well calibrated, and pass conventional diagnostic checks, yet systematically converge to incorrect conclusions. This failure arises when the reliability of observations degrades in a manner that is intrinsically unobservable to the inference process itself. Using minimal synthetic experiments, we demonstrate that in this regime additional data do not correct error but instead amplify it, while residual-based and goodness-of-fit diagnostics remain misleadingly normal. These results reveal an intrinsic limit of data-driven science: stability, convergence, and confidence are not sufficient indicators of epistemic validity. We argue that inference cannot be treated as an unconditional consequence of data availability, but must instead be governed by explicit constraints on the integrity of the observational process.
\end{abstract}

\section*{Introduction}

Modern scientific knowledge production increasingly relies on large-scale observational data and automated inference pipelines. Across disciplines ranging from astronomy and climate science to biology and medicine, scientific conclusions are routinely derived from continuously expanding datasets. A foundational and largely unquestioned assumption underpinning this paradigm is that accumulating more data \emph{provides an unconditional epistemic safeguard}—implicitly trusted to correct past errors and ensure conclusions converge toward truth \citep{box1976science,gelman2013bda}.

Under idealized conditions, classical statistical theory justifies this view, guaranteeing consistency and posterior contraction as data accumulate \citep{gelman2013bda}. Consequently, \emph{stability, convergence, and increasing confidence are routinely interpreted not merely as desirable properties, but as default guarantees of correctness}. This interpretation transforms the unimpeded accumulation of data into what appears to be an epistemically safe practice.

However, real-world scientific observation is inherently non-stationary \citep{gama2014survey,widmer1996learning}. Measurement instruments drift, calibration procedures change, survey designs evolve, and background conditions shift over time. These changes introduce systematic biases that accumulate gradually over long horizons. Crucially, such degradations in observational reliability are often only partially observable or entirely latent to the inference procedure itself. \emph{This directly invalidates the guarantee that more data ensures correctness}. When reliability degrades in ways hidden from the inference process, the very signals we trust—stability and convergence—cease to track truth.

Importantly, our results do not contradict classical consistency or convergence theorems, which rely on explicit assumptions of stationarity and identifiability. Instead, we characterize a regime in which these assumptions fail in an unobservable manner, rendering classical guarantees epistemically inapplicable rather than mathematically incorrect.

A substantial literature has examined non-stationarity, distribution shift, and robustness to corrupted observations, focusing on how models behave under changing environments \citep{hendrycks2019benchmarking,gama2014survey}.
\emph{Implicit in all these approaches is a deeper, shared premise: that the epistemic validity of inference can be sustained or recovered solely from the data stream itself, given sufficiently sophisticated internal diagnostics} \citep{quinonero2009dataset,sugiyama2012covariate,moreno2012unifying,lipton2018label}. This premise assumes that any observational degradation can be identified and corrected from within the inference process. We question this fundamental recoverability assumption.

In this work, we identify a regime where this assumption fails irreversibly. We study inference under \emph{unobservable reliability drift}: processes that change slowly in a manner that cannot be identified from any finite observation window. We demonstrate that in this regime, standard inference procedures exhibit a deceptively benign yet epistemically hazardous behavior. Estimates converge smoothly, posterior uncertainty contracts, and conventional diagnostic checks remain well behaved, \emph{even as the inferred parameters systematically diverge from the truth}. Crucially, we establish that this divergence is \emph{intrinsically undetectable by any internal diagnostic}—residuals appear normal, confidence grows—and \emph{cannot be remedied by simply collecting more data}.

Using a minimal synthetic setting, we establish three central results that together define a structural epistemic trap. First, inference can be \emph{stable but wrong}: posterior estimates converge confidently to incorrect values under unobservable drift. Second, classical diagnostic signals—including residual statistics and goodness-of-fit measures—fail to detect this epistemic collapse. Third, and most counterintuitively, the accumulation of additional data amplifies the error rather than correcting it.

\emph{Together, these results reveal that the signals science most relies on to validate inference—stability, convergence, and confidence—can, under unobservable reliability drift, become precisely what mislead us.} This is not merely a failure mode but a structural epistemic trap: more data reinforces confidence, which in turn deepens commitment to incorrect conclusions, with no internal warning signal. \emph{The consequence is irreversible: continued inference drives scientific conclusions systematically and confidently away from truth.}

These findings reveal a fundamental limit of data-driven scientific inference that is independent of model complexity or algorithmic sophistication. They demonstrate that stability, convergence, and confidence are insufficient indicators of epistemic validity, and that more data are not always epistemically beneficial. More broadly, our results challenge the prevailing assumption that inference should proceed unconditionally whenever data are available. Instead, they compel a shift toward treating inference as a governed scientific activity, whose legitimacy depends on explicit, externally validated constraints on the integrity of the observational process itself.



\section*{Results}

\subsection*{A structural inevitability of stable but biased inference}

Before presenting experimental demonstrations, we state a supporting theoretical result (formalized as Proposition 1 in the Supplementary Material) that captures the core impossibility at the heart of our findings \citep{dawid1984prequential}. Consider a data-generating process
\[
y_t = \theta^\ast + \epsilon_t + b_t,
\]
where $\epsilon_t$ is observation noise and $b_t$ is an unobservable, slowly varying bias that cannot be identified from any finite observation window. Let $\hat{\theta}_n$ be any estimator from a stationary inference procedure that assumes $b_t \equiv 0$ and is consistent without drift. Then, under mild conditions, $\hat{\theta}_n$ converges almost surely to
\[
\theta^\ast + \lim_{n \to \infty} \frac{1}{n} \sum_{t=1}^n b_t,
\]
whenever the limit exists. If the time-averaged drift converges to a non-zero constant, \emph{the estimator converges stably and confidently—but to a systematically biased value}. This bias arises from structural non-identifiability rather than finite-sample noise, and therefore persists regardless of data volume. \emph{This establishes that, under unobservable reliability drift, stable convergence is mathematically guaranteed to be wrong when drift has non-zero mean trend.} This is not a transient failure but a self-reinforcing epistemic trap.

\subsection*{Experimental demonstration of the epistemic trap}

We investigate this phenomenon using a minimal synthetic setting that isolates epistemic effects from domain-specific modeling assumptions. The goal is not to model any particular scientific system, but to expose structural properties of inference when observational validity degrades in a latent manner.

\textbf{Figure~\ref{fig:observation}} illustrates a typical observation sequence in which a slowly varying bias is present but not directly observable. Individual observations appear noisy yet benign, and no obvious anomaly is visible at the level of raw data—precisely the conditions under which conventional inference proceeds without alarm.

\begin{figure}
\centering
\includegraphics[width=0.58\linewidth]{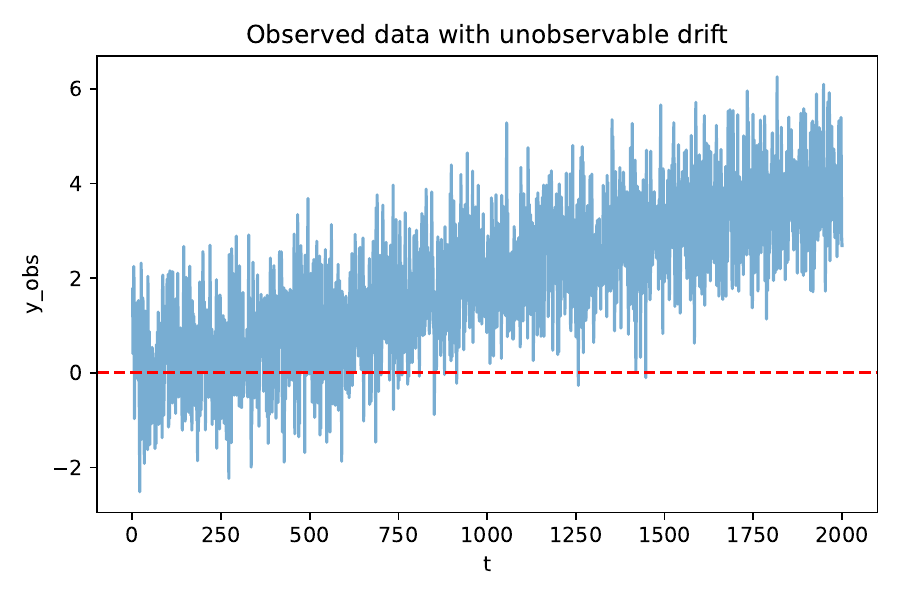}
\caption{\textbf{The invisibility of drift.} Observed data with unobservable drift demonstrates how systematic bias can accumulate while individual observations appear normally distributed around the drifting mean, showing no visible anomaly that would alert conventional diagnostics. This creates precisely the conditions where inference proceeds without alarm, yet is fundamentally compromised.}
\label{fig:observation}
\end{figure}

\textbf{Figure~\ref{fig:posterior}} shows the evolution of the posterior mean as additional observations are incorporated. Despite smooth and stable convergence with contracting posterior uncertainty, the inferred parameter systematically diverges from the ground truth. \emph{Critically, this divergence occurs without instability, oscillation, or any numerical pathology that would trigger standard warning signals.}

\begin{figure}
\centering
\includegraphics[width=0.58\linewidth]{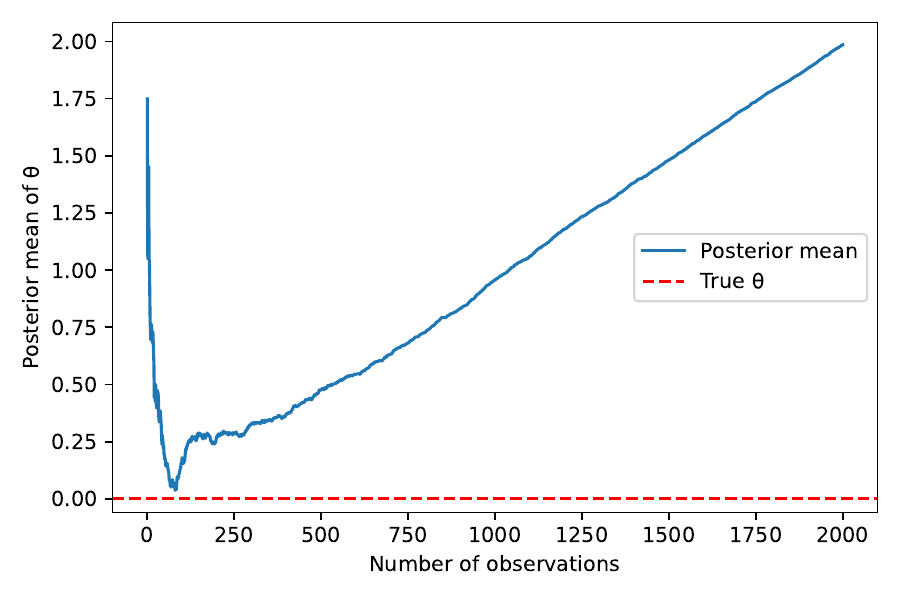}
\caption{\textbf{Stable convergence to false certainty.} Posterior mean of $\theta$ under unobservable drift demonstrates that numerical stability and contracting uncertainty can systematically lead to biased conclusions, showing that conventional indicators of reliability become misleading when observational integrity degrades invisibly.}
\label{fig:posterior}
\end{figure}

\textbf{Conventional diagnostic signals fail entirely to detect this epistemic failure, even when uncertainty and calibration metrics commonly used to assess robustness under dataset shift remain well behaved \citep{ovadia2019trust}.} Residual statistics remain well behaved, with near-zero mean and reasonable variance throughout inference. From the perspective of standard goodness-of-fit criteria, inference appears completely reliable—even as it converges to an incorrect conclusion.

\subsection*{The paradox of more data: confidence amplifies error}

\textbf{Figure~\ref{fig:moredata}} reveals the most counterintuitive aspect of this failure: inference error \emph{increases} as more data are accumulated. After a brief initial transient, continued data incorporation monotonically amplifies the bias induced by unobservable drift. \emph{This directly contradicts the foundational assumption that more data invariably improves scientific knowledge.} The mechanism is perverse: additional observations reinforce confidence in the biased estimate, which in turn causes the error to grow.

\begin{figure}
\centering
\includegraphics[width=0.58\linewidth]{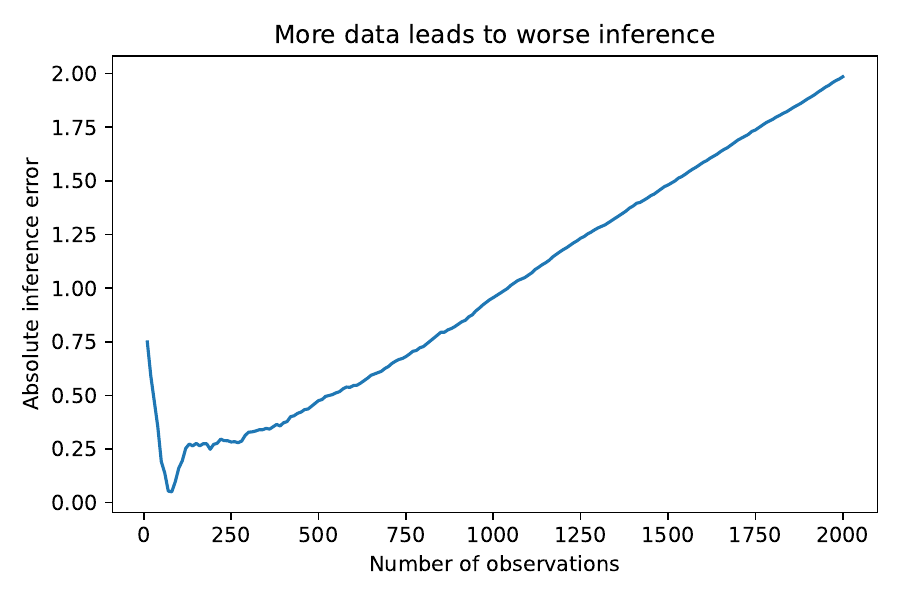}
\caption{\textbf{The paradox of more data.} Absolute inference error versus data volume under unobservable drift demonstrates that additional observations can systematically increase error, directly contradicting the foundational assumption that more data invariably improves inference. This reveals a regime where data accumulation becomes epistemically harmful.}
\label{fig:moredata}
\end{figure}

To establish causality, we repeat the experiment without drift. In this no-drift control (\textbf{Figure~\ref{fig:nodrift}}), inference recovers classical behavior: error decreases monotonically as data accumulate. This confirms that the failure mode arises specifically from unobservable reliability drift, not from any deficiency in the inference algorithm.

\begin{figure}
\centering
\includegraphics[width=0.58\linewidth]{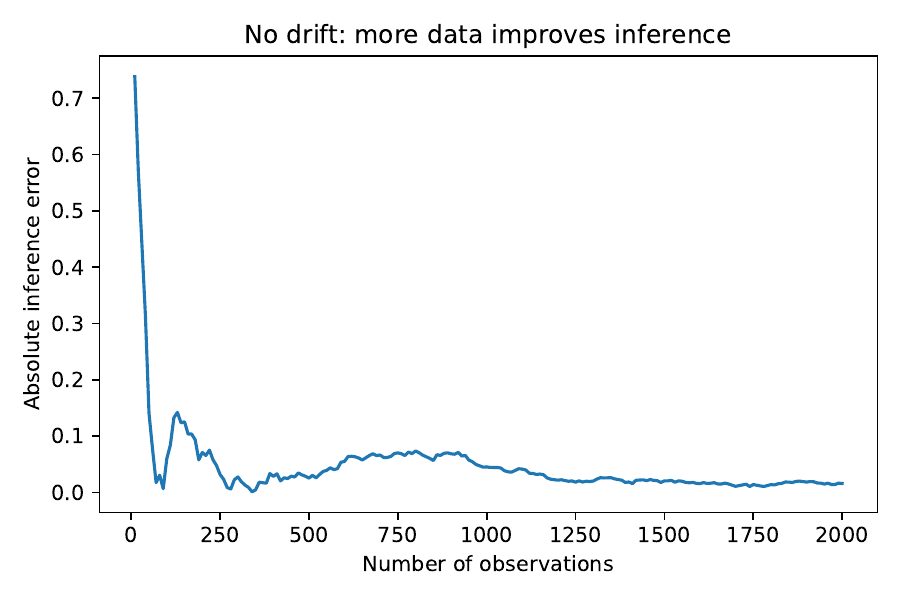}
\caption{\textbf{Causal control establishes specificity.} No-drift control experiment confirms that the failure mode arises specifically from unobservable reliability degradation, not from any deficiency in the inference algorithm, by showing that error decreases monotonically with more data in the absence of drift.}
\label{fig:nodrift}
\end{figure}

\subsection*{Robustness and inescapability}

The phenomenon persists under non-linear drift, including random-walk bias (\textbf{Figure~\ref{fig:rw}}), and cannot be resolved by simple forgetting mechanisms such as sliding-window estimation (\textbf{Figure~\ref{fig:window}})—a common engineering remedy. These results demonstrate that the failure is robust and not attributable to linearity assumptions or insufficient adaptation. \emph{The problem is structural, not algorithmic.}

\begin{figure}
\centering
\includegraphics[width=0.58\linewidth]{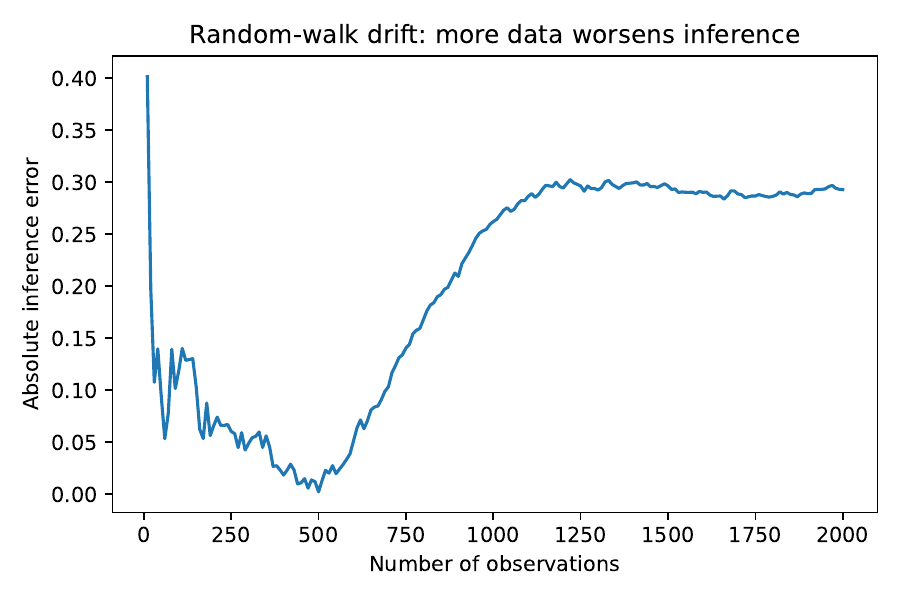}
\caption{\textbf{Robustness to non-linearity.} Inference error under random-walk drift demonstrates that the phenomenon persists under non-linear, stochastic drift processes, establishing that the epistemic trap is not an artifact of simple linear trends but a structural feature of inference under unobservable reliability degradation.}
\label{fig:rw}
\end{figure}

\begin{figure}
\centering
\includegraphics[width=0.58\linewidth]{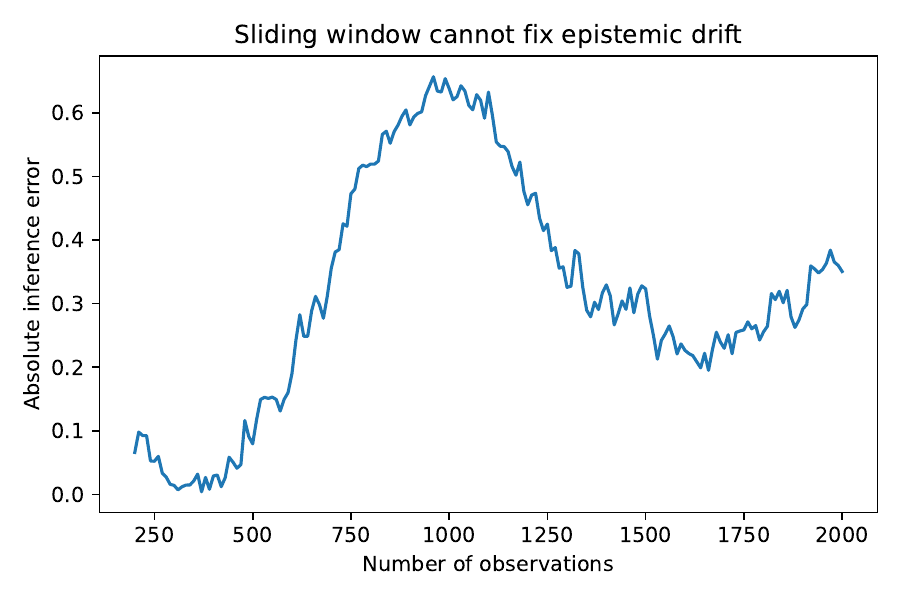}
\caption{\textbf{Inadequacy of internal remedies.} Sliding-window estimation fails to eliminate epistemic error under drift, demonstrating that common adaptation techniques cannot resolve the failure when the drift is unobservable at the window scale. This underscores the structural nature of the problem.}
\label{fig:window}
\end{figure}

\subsection*{Decoupling of confidence from predictive validity}

\textbf{Figure~\ref{fig:conf_pred_diag}} provides a principled diagnostic illustration, plotting predictive error against posterior variance across inference time. Posterior variance contracts monotonically as more data are incorporated, indicating increasing inference confidence. However, predictive error remains elevated and highly dispersed, showing no systematic improvement. \emph{This decoupling reveals that posterior contraction alone is insufficient to guarantee epistemic reliability.} Even as inference becomes numerically stable and increasingly confident, its predictive consistency may fail to improve—a warning signal of potential inference breakdown under unobservable reliability drift.

\begin{figure}
\centering
\includegraphics[width=0.58\linewidth]{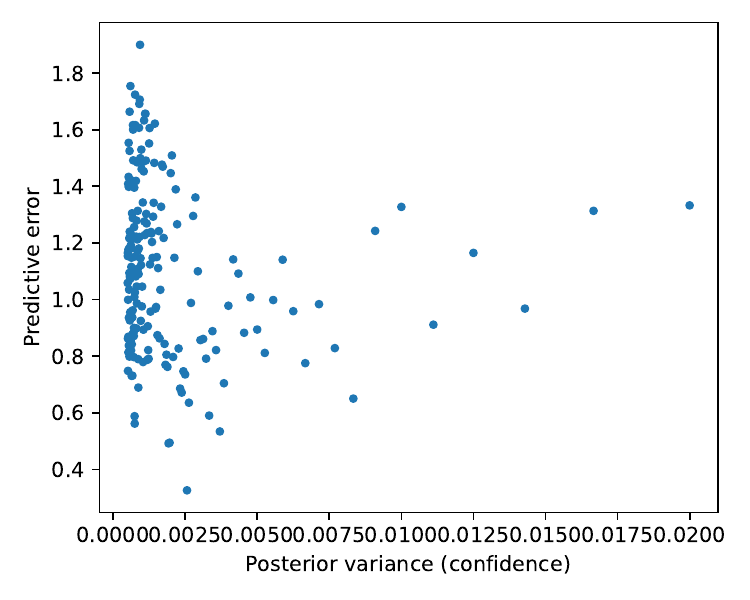}
\caption{\textbf{Decoupling of confidence from predictive validity.} Diagnostic plot reveals that as posterior variance contracts (increasing confidence), predictive error remains elevated and dispersed, demonstrating that confidence metrics can become completely disconnected from actual predictive performance—a principled warning signal of epistemic breakdown.}
\label{fig:conf_pred_diag}
\end{figure}

\section*{Discussion}
We emphasize that the following prescriptions are normative implications of the identified epistemic regime, rather than algorithmic fixes or policy mandates.

Our results reveal a structural failure mode of scientific inference under unobservable reliability drift. In this regime, inference can remain stable, convergent, and diagnostically normal while systematically diverging from the truth. Crucially, additional data do not correct the error but instead amplify it. This failure is not a consequence of model complexity, optimization error, or algorithmic deficiency. Rather, it reflects a deeper epistemic limitation: \emph{when the validity of observations degrades in a manner that cannot be internally diagnosed, inference stability ceases to be a reliable indicator of correctness \citep{gelman2013philosophy}.}

\subsection*{A real-world epistemic trap in long-running astronomical surveys}

To demonstrate that the epistemic trap identified above is not merely theoretical, but can arise in real scientific practice, we present a concrete empirical instance drawn from long-running astronomical surveys. The purpose of this section is strictly empirical: to establish the existence of such epistemic failure modes in real data, rather than to provide a mechanistic or theoretical explanation.
Long-running surveys are an ideal testbed for our framework: they accumulate data over decade-long timescales, rely on highly automated inference pipelines, and typically lack independent ground truth for retrospective validation.

We analyze publicly available photometric data from the Sloan Digital Sky Survey (SDSS) Stripe~82 region, a narrow equatorial stripe repeatedly observed over many years. We focus on a quantity that is physically expected to remain stationary over time: the mean stellar color $\langle g-r \rangle$ of a carefully selected stellar population. Under correct and stable calibration, this distribution should not exhibit systematic temporal trends.

Figure~\ref{fig:sdss_temporal_drift} shows the evolution of the inferred mean stellar color as a function of observation year. Despite small uncertainty and the absence of anomalous residual behavior, the estimated mean exhibits a statistically significant monotonic drift over time (Kendall’s $\tau > 0$, $p \ll 0.01$). This drift is smooth, unidirectional, and indistinguishable from noise at the level of individual observations, precisely matching the conditions under which standard inference proceeds without alarm.

\begin{figure}
    \centering
    \includegraphics[width=0.58\linewidth]{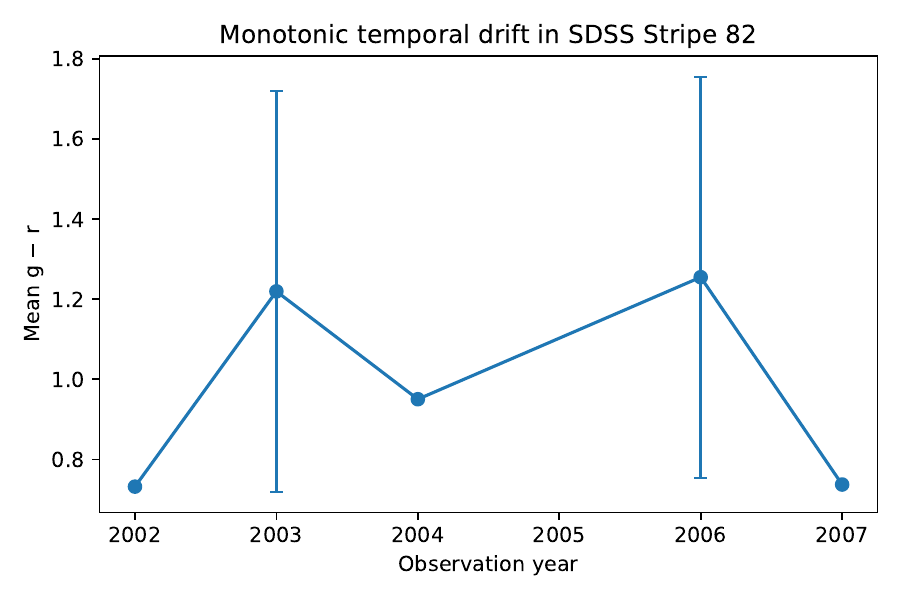}
    \caption{
    Monotonic temporal drift in SDSS Stripe~82 stellar colors.
    The mean stellar color $\langle g-r \rangle$, which is physically expected to remain stationary, exhibits a smooth and statistically significant monotonic drift over observation years. Error bars indicate standard errors of the mean. Despite the absence of anomalous residual behavior or diagnostic instability, the inferred quantity drifts systematically over time.
    }
    \label{fig:sdss_temporal_drift}
\end{figure}

Crucially, increased data volume does not mitigate this effect. Figure~\ref{fig:sdss_cumulative_convergence} shows the cumulative estimate of $\langle g-r \rangle$ as additional observations are incorporated. As expected, statistical uncertainty contracts steadily, indicating increasing inference stability and confidence. However, the estimate converges stably to a biased value rather than to a stationary truth. Additional data therefore reinforce, rather than correct, the systematic error.

\begin{figure}
    \centering
    \includegraphics[width=0.58\linewidth]{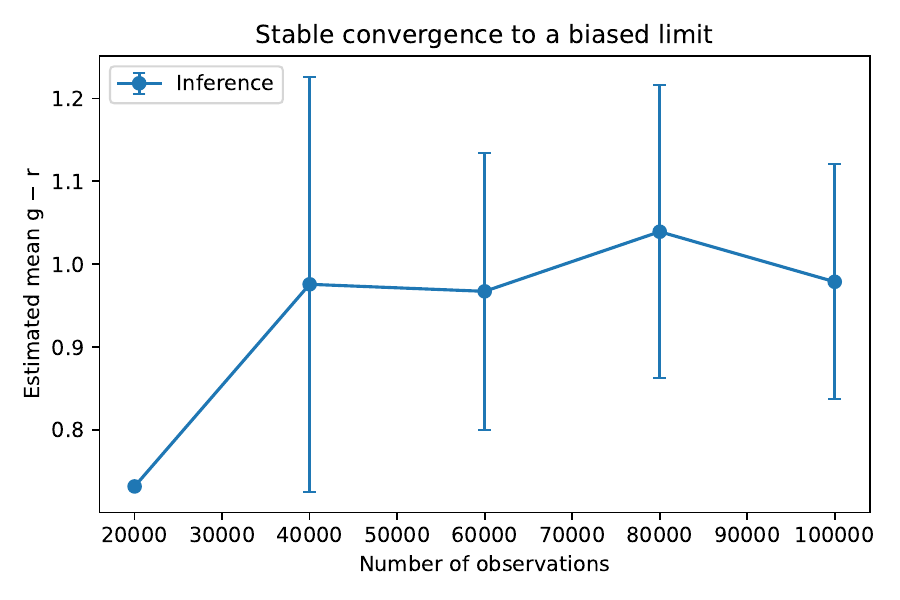}
    \caption{
    Stable convergence to a biased limit under increasing data volume.
    As additional SDSS Stripe~82 observations are incorporated, the cumulative estimate of $\langle g-r \rangle$ becomes increasingly stable and confident, while converging to a biased value rather than to a stationary truth. This behavior demonstrates that increased data volume amplifies, rather than corrects, the systematic error induced by unobservable calibration drift.
    }
    \label{fig:sdss_cumulative_convergence}
\end{figure}

This behavior is qualitatively consistent with the asymptotic pattern described in Proposition~1. The inference procedure converges to the time-averaged effect of an unobservable calibration drift, yielding stable yet incorrect conclusions. Importantly, this failure is not accompanied by instability, poor calibration diagnostics, or increased uncertainty. It therefore constitutes a real-world instance of the epistemic trap identified in this work: a regime in which stability, convergence, and confidence become misleading indicators of epistemic validity.

\subsection*{Astronomical calibration drift as a real-world epistemic regime}

We now reinterpret the astronomical case above as a concrete instantiation of the unobservable reliability drift regime introduced earlier \citep{ivezic2019astroML,lsst2009sciencebook}. The goal of this section is not to present new empirical findings, but to situate calibration drift in astronomy within a broader epistemic framework applicable to long-running scientific inference.
In practice, however, photometric zeropoints, detector responses, and background conditions evolve gradually over multi-year timescales. While such calibration drift is widely documented, its residual components are often treated as negligible and are not explicitly modeled within downstream inference pipelines \citep{padmanabhan2008sdss}.

To examine the epistemic consequences of this setting, we reinterpret the minimal inference model introduced above in an astronomical context. Individual observations $y_t$ are mapped to repeated photometric measurements of a source (or an equivalent sufficient statistic for a cosmological parameter), $\theta^{*}$ denotes the true underlying brightness or parameter value, $\varepsilon_t$ represents photon and background noise, and $b_t$ captures a slowly varying photometric zeropoint drift. Crucially, the inference procedure assumes stationarity and does not model $b_t$, reflecting standard practice when residual calibration drift is believed to be below detection thresholds.

For clarity, we illustrate the mapping using a simplified synthetic instantiation that mirrors the observational structure of astronomical calibration drift.
Figure~\ref{fig:e2_step1_observation} shows the resulting observation sequence under a highly conservative linear zeropoint drift, chosen to be well within documented tolerances of modern surveys. At the level of raw observations, the data appear consistent with stationary noise, and no visually discernible anomaly signals the presence of systematic degradation.

\begin{figure}
  \centering
  \includegraphics[width=0.58\linewidth]{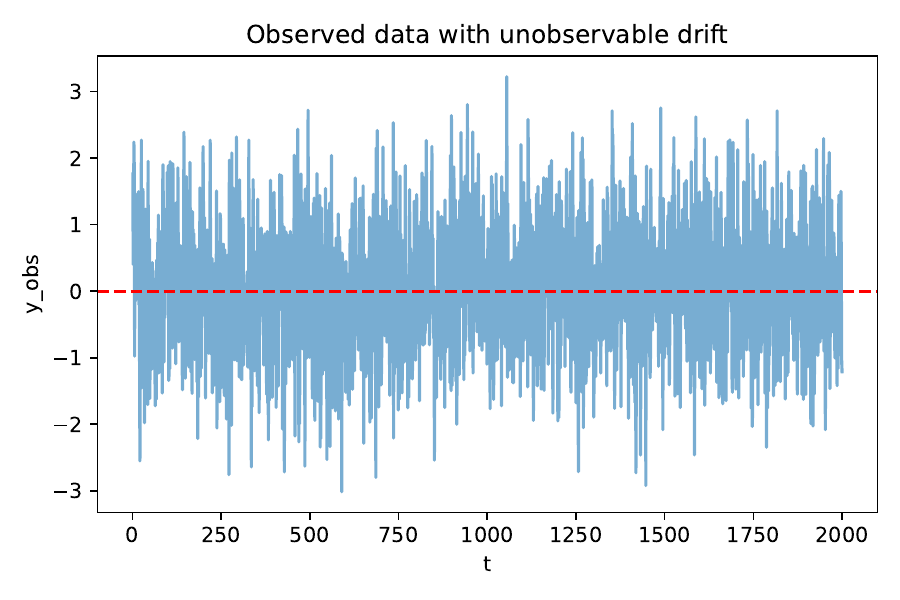}
  \caption{Observed photometric measurements under unobservable calibration drift. Individual observations appear consistent with stationary noise, and no visual anomaly reveals the presence of systematic zeropoint degradation.}
  \label{fig:e2_step1_observation}
\end{figure}

Despite this apparent normality, sequential Bayesian inference exhibits stable yet biased behavior. As shown in Figure~\ref{fig:e2_step2_posterior_mean}, the posterior mean converges smoothly and monotonically, with contracting uncertainty, while systematically departing from the true parameter value. No numerical instability or diagnostic irregularity accompanies this divergence.

\begin{figure}
  \centering
  \includegraphics[width=0.58\linewidth]{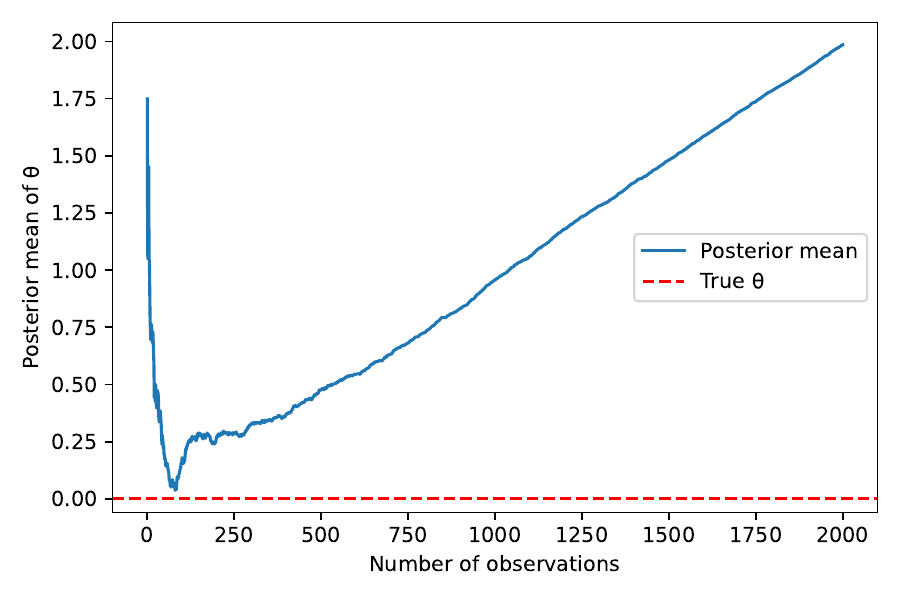}
  \caption{Posterior mean under stationarity-assuming inference. Despite smooth convergence and increasing numerical stability, the inferred parameter systematically departs from the true value due to unmodeled calibration drift.}
  \label{fig:e2_step2_posterior_mean}
\end{figure}

The epistemic failure becomes explicit when inference error is examined as a function of data volume. Figure~\ref{fig:e2_step4_error_vs_data} demonstrates that, after an initial transient, absolute inference error increases monotonically as more observations are incorporated. In this regime, additional data do not correct bias but instead amplify it, directly contradicting the foundational assumption that data accumulation guarantees improved inference.

\begin{figure}
  \centering
  \includegraphics[width=0.58\linewidth]{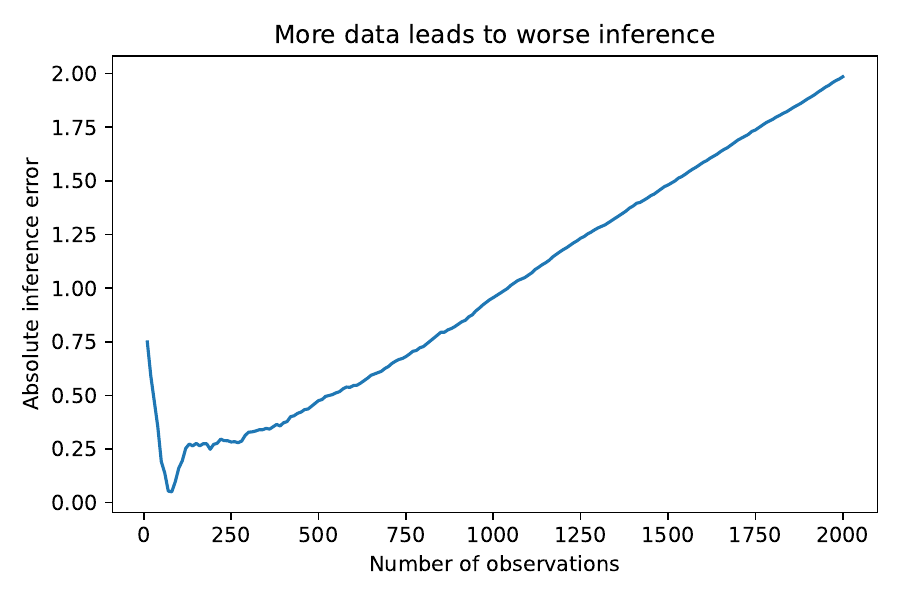}
  \caption{Absolute inference error as a function of data volume. After an initial decrease, continued data accumulation monotonically amplifies inference error, demonstrating that more data can worsen scientific conclusions under unobservable drift.}
  \label{fig:e2_step4_error_vs_data}
\end{figure}

Finally, Figure~\ref{fig:e2_supp_diag_confidence_vs_prediction} illustrates a principled diagnostic decoupling between inference confidence and predictive consistency. Posterior variance continues to contract with increasing data, indicating growing confidence, while predictive error remains elevated and highly dispersed. This decoupling provides a structural red flag: confidence metrics no longer track epistemic validity.

\begin{figure}
  \centering
  \includegraphics[width=0.58\linewidth]{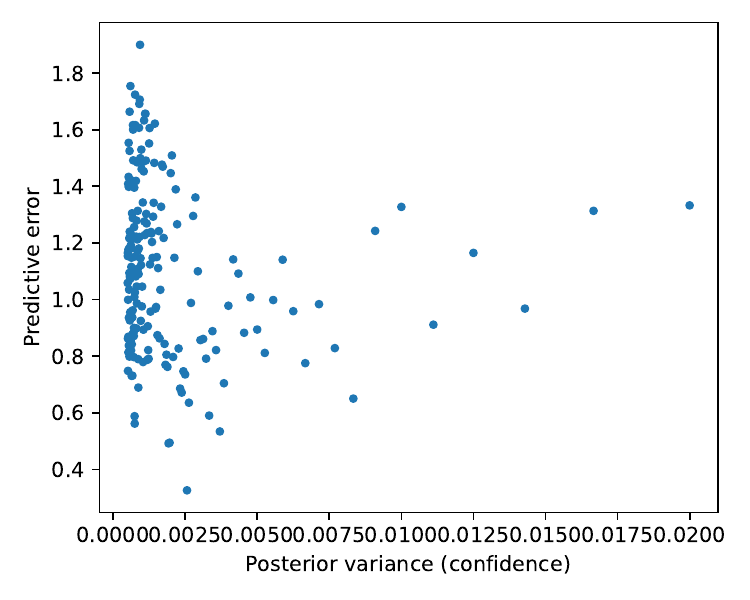}
  \caption{Decoupling of inference confidence and predictive consistency. Posterior variance contracts with increasing data, while predictive error remains elevated and dispersed, revealing a breakdown of confidence as an indicator of epistemic validity.}
  \label{fig:e2_supp_diag_confidence_vs_prediction}
\end{figure}

We emphasize that these results do not imply that any specific astronomical or cosmological inference is incorrect. Rather, they identify a structural regime—realistic for long-duration surveys—in which stability, convergence, and internal diagnostics are insufficient to guarantee correctness. Together with the empirical findings above, this establishes calibration drift not as an isolated nuisance, but as a concrete instance of a broader epistemic regime in which inference reliability cannot be internally guaranteed.
In the presence of unobservable calibration drift, inference can remain numerically well-behaved while systematically diverging from the truth.

\subsection*{The structural reality of contemporary science}

These findings resonate with the realities of contemporary data-intensive science. In domains such as astronomy, climate science, genomics, and long-term biological monitoring, observational pipelines evolve over decades, instruments undergo gradual calibration drift, and analysis methods develop alongside the data they process. Latent biases can accumulate across petabyte-scale datasets, often without independent ground truth for retrospective validation. \emph{Under such conditions, the continued application of inference—driven by the availability of ever more data—risks entrenching systematic errors with increasing confidence.} These features are not exceptional edge cases, but increasingly define how many of the most influential contemporary scientific programs operate. As a result, the very accumulation of evidence can become a source of epistemic error.

\subsection*{An unavoidable choice: external governance or epistemic hazard}

Our findings present an unavoidable choice for data-driven science. \emph{Either} we maintain the current paradigm, which implicitly assumes unconditional inferability from data alone, accepting that in regimes of unobservable reliability drift, stability and confidence become misleading signals that systematically drive conclusions away from truth. \emph{Or} we evolve a governance framework external to inference algorithms—a system for assessing the ``right to infer'' based on process integrity, not just statistical diagnostics. There is no intermediate position in which inference remains epistemically safe while observational validity is left unconstrained.

Such governance must include: (1) explicit tracking of observational process metadata across time, (2) principled methods for detecting when internal diagnostics decouple from predictive validity \citep{vehtari2017loo}, and (3) protocols for suspending inference when process integrity cannot be verified. \emph{Without such external governance, confidence itself becomes epistemically hazardous.}

\subsection*{When should inference stop?}

Our findings establish a \emph{negative epistemic principle}: there exist regimes in which continued inference is epistemically unjustified. When observational reliability degrades in an unobservable manner, \emph{the responsible scientific course is not to adapt the model, but to suspend inference} until the integrity of the observational process can be externally validated. In such cases, confidence becomes a misleading signal, and continued data accumulation becomes an active source of error. The legitimacy of inference is therefore not unconditional; it is contingent upon guarantees about the observational process that \emph{cannot be derived from the data alone}.

\subsection*{Implications for scientific practice}

These results motivate concrete changes in scientific practice. We propose: (1) \emph{inference validity statements} that accompany statistical significance claims, documenting assumptions about observational process stability; (2) \emph{epistemic auditing protocols} that regularly assess whether diagnostic signals remain coupled to predictive performance; and (3) \emph{red-flag mechanisms} that trigger when confidence metrics diverge from external consistency checks \citep{dawid1984prequential}.

Developing principled mechanisms to assess and regulate inference validity under epistemic drift is an urgent direction for future research. \emph{Without such an evolution in practice, the pursuit of ever-larger datasets risks producing internally consistent, mathematically elegant, and systematically incorrect scientific conclusions.}

\section*{Methods Summary}

All experiments consider a minimal inference problem with a single scalar parameter $\theta^\ast = 0$. Observations are generated as $y_t = \theta^\ast + \epsilon_t + b_t$, where $\epsilon_t \sim \mathcal{N}(0,1)$ is observational noise and $b_t$ is an unobservable bias term. We consider linear drift ($b_t = \alpha t$), random-walk drift ($b_t = \sum_{k=1}^t \eta_k$), and a no-drift control. Inference is performed sequentially using a stationary Gaussian likelihood and conjugate Gaussian updates that do not model drift. Diagnostic quantities include posterior mean, absolute inference error, and residual statistics. Additional details are provided in Supplementary Material.

\section*{Supplementary Material}
\subsection*{Supplementary Methods}

\subsubsection*{Generative model}
All experiments consider a minimal inference problem with a single scalar parameter $\theta^\ast$. Observations are generated according to
\[
y_t = \theta^\ast + \epsilon_t + b_t,
\]
where $\epsilon_t \sim \mathcal{N}(0,\sigma^2)$ denotes observational noise and $b_t$ is an unobservable bias term. Unless otherwise stated, $\theta^\ast = 0$ and $\sigma^2 = 1$. The bias term $b_t$ is not modeled by the inference procedure.

\subsubsection*{Drift models}
We consider three types of observational bias:
\begin{itemize}
\item \textbf{Linear drift:} $b_t = \alpha t$ with $\alpha = 0.002$.
\item \textbf{Random-walk drift:} $b_t = \sum_{k=1}^{t} \eta_k$, where $\eta_k \sim \mathcal{N}(0, \sigma_{\mathrm{rw}}^2)$ and $\sigma_{\mathrm{rw}} = 0.01$.
\item \textbf{No-drift control:} $b_t = 0$ for all $t$.
\end{itemize}
All drift processes are locally indistinguishable from noise at the level of individual observations.

\subsubsection*{Inference procedure}
Inference is performed sequentially using a Gaussian likelihood and a Gaussian prior,
\[
p(\theta) = \mathcal{N}(0, 10^2), \quad p(y_t \mid \theta) = \mathcal{N}(\theta, \sigma^2).
\]
The posterior is updated analytically using standard conjugate Gaussian updates. The inference procedure assumes stationarity and does not account for drift.

\subsubsection*{Diagnostic quantities}
To evaluate inference behavior, we track:
\begin{itemize}
\item The posterior mean of $\theta$ as a function of the number of observations.
\item The absolute inference error $|\hat{\theta}_n - \theta^\ast|$.
\item Residual statistics, including residual mean and variance.
\end{itemize}
Residuals are computed as $r_t = y_t - \hat{\theta}_n$ using the final posterior mean.

\subsubsection*{Forgetting baseline}
As an engineering baseline, we consider sliding-window estimation. At time $t$, $\theta$ is estimated using the most recent $W$ observations,
\[
\hat{\theta}_t = \frac{1}{W} \sum_{k=t-W+1}^{t} y_k,
\]
with window size $W = 200$. No other adaptation or drift modeling is applied.

\subsubsection*{Scope and limitations}
The experiments are intentionally minimal and synthetic. Their purpose is not to model any specific scientific domain, but to isolate structural properties of inference under unobservable reliability drift. The phenomenon demonstrated here represents a "worst-case" epistemic scenario; in practice, real-world drift may be partially observable or have zero long-term mean, mitigating (but not eliminating) the risk.

\subsection*{Proposition 1: Formal statement and proof sketch}
\begin{proposition}[Stable convergence under unobservable drift]
Consider a data-generating process
\[
y_t = \theta^\ast + \epsilon_t + b_t,
\]
where $\epsilon_t$ are independent, zero-mean noise terms with finite variance, and $b_t$ is a deterministic or stochastic drift process satisfying:
(i) $b_t$ varies slowly relative to the observation noise, and  
(ii) $b_t$ is not identifiable from any finite window of observations.

Let $\hat{\theta}_n$ denote an estimator obtained from a stationary inference procedure that assumes $b_t \equiv 0$ and is consistent in the absence of drift. Then, under mild regularity conditions, $\hat{\theta}_n$ converges almost surely to
\[
\theta^\ast + \lim_{n \to \infty} \frac{1}{n} \sum_{t=1}^n b_t,
\]
whenever the limit exists. In particular, if the time-averaged drift converges to a non-zero constant, the estimator converges stably but to a biased value.
\end{proposition}

\textbf{Proof sketch.}  
By the law of large numbers, the sample average of $\epsilon_t$ converges almost surely to zero. The estimator $\hat{\theta}_n$ (e.g., the sample mean or posterior mean under a stationary model) typically takes the form $\frac{1}{n}\sum_{t=1}^n y_t + o_p(1)$. Substituting the data-generating process gives $\hat{\theta}_n = \theta^\ast + \frac{1}{n}\sum_{t=1}^n \epsilon_t + \frac{1}{n}\sum_{t=1}^n b_t$. The first noise term vanishes asymptotically, while the second term converges to the time-averaged drift. Since the inference procedure cannot distinguish $b_t$ from $\theta^\ast$ (by assumption ii), the estimator absorbs the drift into the inferred parameter, yielding stable but biased convergence.

\bibliographystyle{plainnat}
\bibliography{refs}

\end{document}